\documentclass[runningheads]{llncs}

\usepackage{times}
\usepackage{helvet}
\usepackage{courier}
\usepackage[hyphens]{url}
\usepackage{graphicx}
\usepackage{graphicx}
\usepackage{caption}
\usepackage{subcaption} 
\usepackage[switch]{lineno}  %

\usepackage{xspace}
\usepackage[utf8]{inputenc}
\usepackage{enumitem}
\usepackage{amsmath,amsfonts,amssymb}
\usepackage[table]{xcolor}
\usepackage{stmaryrd}
\usepackage{comment}

\usepackage{adjustbox}         
\usepackage{booktabs,multirow,siunitx}  
\usepackage{mathptmx}          

\usepackage{xspace}
\usepackage{setspace}
\usepackage{empheq}

\usepackage{amsmath,amssymb,amssymb,amsthm,latexsym}
\usepackage{graphicx,verbatim,color,epsfig,graphics,url,multirow}
\usepackage[noend,linesnumberedhidden,ruled,resetcount,noline]{algorithm2e}
\usepackage{algorithmic}
\usepackage{dashbox}
\usepackage{stmaryrd}
\usepackage{longtable}

\usepackage{hyperref}
\hypersetup{colorlinks=true,pdfborder={0 0 0},urlbordercolor={0 0 0},urlcolor=midblue,linkcolor=midblue,citecolor=midblue}

\usepackage{pgf}
\usepackage{tikz}
\usepackage{forest}

\usepackage{soul}

\usetikzlibrary{arrows,decorations.pathmorphing,decorations.footprints,fadings,calc,trees,mindmap,shadows,decorations.text,patterns,positioning,shapes,matrix,fit}
\usetikzlibrary{arrows,shadows,backgrounds}
\usetikzlibrary{arrows.meta}
\usetikzlibrary{positioning,fit}
\usetikzlibrary{automata}
\usetikzlibrary{matrix}
\usetikzlibrary{shapes.symbols,shapes.misc,shapes.arrows}
\usetikzlibrary{matrix,chains,scopes,decorations.pathmorphing}

\DeclareMathAlphabet{\mathcal}{OMS}{cmsy}{m}{n}

\title{On Explaining Decision Trees}
\author{%
  Yacine Izza\inst{1} \and
  Alexey Ignatiev\inst{2} \and
  Joao Marques-Silva\inst{3} 
}
\institute{%
  ANITI, Univ. Toulouse, France \and
  Monash Univ., Australia \and
  ANITI, IRIT, CNRS, France
}

\authorrunning{Izza, Ignatiev and Marques-Silva} 
\titlerunning{Explanations for Decision Trees} 

\setlist{nolistsep}

\makeatletter
\def\thm@space@setup{%
  \thm@preskip=2.25pt
  \thm@postskip=2.25pt}
\makeatother

\makeatletter
\renewenvironment{proof}[1][\proofname]{\par
  \pushQED{\qed}%
  \normalfont \partopsep=1.25pt \topsep=1.25pt
  \trivlist
  \item[\hskip\labelsep
        \itshape
    #1\@addpunct{.}]\ignorespaces
}{%
  \popQED\endtrivlist\@endpefalse
}
\makeatother

\makeatletter
\renewcommand\paragraph{\@startsection{paragraph}{4}{\z@}%
  {1.75ex \@plus0.5ex \@minus.2ex}%
  {-1em}%
  {\normalfont\normalsize\bfseries}}
\makeatother

\setlist{nosep}

\definecolor{gray}{rgb}{.4,.4,.4}
\definecolor{midgrey}{rgb}{0.5,0.5,0.5}
\definecolor{middarkgrey}{rgb}{0.35,0.35,0.35}
\definecolor{darkgrey}{rgb}{0.3,0.3,0.3}
\definecolor{darkred}{rgb}{0.7,0.1,0.1}
\definecolor{midblue}{rgb}{0.2,0.2,0.7}
\definecolor{darkblue}{rgb}{0.1,0.1,0.5}
\definecolor{darkgreen}{rgb}{0.1,0.5,0.1}
\definecolor{defseagreen}{cmyk}{0.69,0,0.50,0}

\newcommand{\jnoteF}[1]{}

\newcommand{\fml}[1]{{\mathcal{#1}}}

\newcommand{\tn}[1]{\textnormal{#1}}

\newcommand{\mbf}[1]{\ensuremath\mathbf{#1}}
\newcommand{\mbb}[1]{\ensuremath\mathbb{#1}}

\newcommand{\lpr}{(}
\newcommand{\rpr}{)}

\DeclareMathOperator*{\entails}{\vDash}

\DeclareMathOperator*{\limply}{\rightarrow}

\tikzset{
  0 my edge/.style={densely dashed, my edge},
  my edge/.style={-{Stealth[]}},
}

\SetAlgoNoEnd
\SetAlgoNoLine
\SetFillComment
\SetKwBlock{Let}{let}{end}
\SetKwBlock{FBlock}{}{}
\SetKw{KwNot}{not\xspace}
\SetKw{KwAnd}{and\xspace}
\SetKw{KwOr}{or\xspace}
\SetKw{Break}{break\xspace}
\SetKw{Cont}{continue\xspace}
\SetKwData{false}{{\small false}}
\SetKwData{true}{{\small true}}
\SetKwData{st}{\small{\sl st}}
\SetKwData{cores}{$\mathcal{C}$}
\SetKwFunction{SAT}{SAT}
\SetKwBlock{Let}{let}{end}
\SetKwBlock{FBlock}{}{end}
\SetKwInOut{Global}{Global}
\SetKwHangingKw{Algorithm}{Algorithm}
\SetKwHangingKw{Function}{function}
\SetKw{Func}{Function}
\SetKwBlock{Begin}{}{}

\SetAlCapSty{}
\SetAlCapSkip{0.5em}

\newcommand{\BotBlankLine}{\vspace*{1.5pt}}

\newcommand*\patchAmsMathEnvironmentForLineno[1]{%
  \expandafter\let\csname old#1\expandafter\endcsname\csname #1\endcsname
  \expandafter\let\csname oldend#1\expandafter\endcsname\csname end#1\endcsname
  \renewenvironment{#1}%
     {\linenomath\csname old#1\endcsname}%
     {\csname oldend#1\endcsname\endlinenomath}}%
\newcommand*\patchBothAmsMathEnvironmentsForLineno[1]{%
  \patchAmsMathEnvironmentForLineno{#1}%
  \patchAmsMathEnvironmentForLineno{#1*}}%
\AtBeginDocument{%
\patchBothAmsMathEnvironmentsForLineno{equation}%
\patchBothAmsMathEnvironmentsForLineno{align}%
\patchBothAmsMathEnvironmentsForLineno{flalign}%
\patchBothAmsMathEnvironmentsForLineno{alignat}%
\patchBothAmsMathEnvironmentsForLineno{gather}%
\patchBothAmsMathEnvironmentsForLineno{multline}%
}

\begin{document}

\maketitle
\setcounter{footnote}{0}

\begin{abstract}
  Decision trees (DTs) epitomize what have become to be known as
  interpretable machine learning (ML) models.
  This is informally motivated by paths in DTs being often much
  smaller than the total number of features.
  This paper shows that in some settings DTs can hardly be deemed
  interpretable, with paths in a DT being arbitrarily larger than a 
  PI-explanation, i.e.\ a subset-minimal set of feature values that
  entails the prediction.
  As a result, the paper proposes a novel model for computing
  PI-explanations of DTs, which enables computing one PI-explanation
  in polynomial time. Moreover, it is shown that enumeration of
  PI-explanations can be reduced to the enumeration of minimal hitting
  sets.
  %
  Experimental results were obtained on a wide range of publicly
  available datasets with well-known DT-learning tools, and confirm
  that in most cases DTs have paths that are proper supersets of
  PI-explanations.
\end{abstract}

\section{Introduction} \label{sec:intro}

Decision trees (DTs) are well known machine learning (ML) models,
studied since at least the
1970s~\cite{garey-sjam72,rivest-ipl76,breiman-bk84,quinlan-ml86,rivest-ml87}.
%
DTs embody what is widely regarded as an interpretable ML
model~\cite{freitas-sigkdd13,guestrin-corr16,muller-dsp18,muller-bk19,molnar-bk19,miller-aij19,pedreschi-acmcs19,zhu-nlpcc19,gombolay-aistats20}\footnote{
  Interpretability is generally accepted to be a subjective concept,
  without a rigorous definition~\cite{lipton-cacm18}. In this paper
  we measure interpretability in terms of the overall succinctness of
  the information provided by a model 
  to explain a given prediction.
  The association of DTs with interpretability can be traced at least
  to Breiman~\cite{breiman-ss01}, who summarizes the interpretability
  of DTs as follows: ``\emph{On interpretability, trees rate an A+}''.
}.
Motivated by this perception, there has been extensive work on
learning DTs (and related logic ML models) with properties deemed
important for interpretability, e.g.\ number of nodes, maximum/average
depth,
etc.~\cite{schaus-ijcai20a,schaus-ijcai20b,hebrard-ijcai20,janota-sat20,schaus-aaai20,avellaneda-aaai20,rudin-naturemi19,holzinger-dmkd19,rudin-nips19,verwer-aaai19,schaus-bnaic19,rudin-mpc18,nipms-ijcai18,ipnms-ijcar18,rudin-aistats18,rudin-jmlr17a,rudin-jmlr17b,leskovec-corr17,rudin-kdd17,rudin-icml17,verwer-cpaior17,bertsimas-ml17,leskovec-kdd16,rudin-icdm16,rudin-aistats15,rudin-aaai13,nijssen-dmkd10,bessiere-cp09,nijssen-kdd07}%
\footnote{The association of DTs with interpretability is also
  illustrated by Interpretable AI 
  (\url{https://www.interpretable.ai/}), which offers interpretability 
  solutions based on optimal decision trees.}.
Moreover, there has been work on \emph{distilling} or approximating
complex ML models with (soft) decision
trees~\cite{hinton-cexaiia17,bastani-corr17a,bastani-corr17b,doshi-velez-aaai18,doshi-velez-corr19,veloso-corr19,doshi-velez-aaai20}. %
Nevertheless, recent work highlights that interpretability should
correlate with how shallow DTs are~\cite{lipton-cacm18,muller-dsp18}.

In contrast with earlier work, this paper investigates the limits of 
interpretability of DTs. Concretely, the paper proposes Boolean
functions for which a minimal DT contains paths with a number of
literals that is arbitrarily larger (growing with the number of
features) than a PI-explanation\footnote{%
  A PI-explanation is a subset-minimal set of feature-value
  pairs that entails the prediction~\cite{darwiche-ijcai18}.} of
constant size.
Experimental results demonstrate that for widely used tools for
constructing DT classifiers, the resulting DTs often contain paths
that are proper supersets of PI-explanations (which we refer as
explanation-redundant paths, in possibly irredundant DTs). Perhaps
more importantly, for a significant number of datasets, and for the
obtained DTs, most of their paths are explanation-redundant. The 
results also indicate that for some DTs, as much as 98\% of feature 
space will be explained by some path that is explanation-redundant.
Motivated by these negative results, we propose a hitting set
formulation for computing PI-explanations of DTs, distinguishing
explanations restricted to literals in a tree path, and explanations 
unrestricted to literals in a tree path. 
In addition, we propose a polynomial time algorithm for computing a
single PI-explanation for any given instance. Finally, the paper
reduces enumeration of PI-explanations of DTs to the problem of
enumerating minimal hitting sets (MHSes), and proposes a solution based
on iterative calls to an NP oracle (e.g.\ SAT or a 0-1 ILP) oracle.
%

%
The paper is structured as follows.~\autoref{sec:prelim}
introduces the notation and definitions used in the remainder of the
paper. \autoref{sec:rdt} studies functions that elicit poor
DT interpretability. \autoref{sec:xdt} proposes a polynomial-time
algorithm for computing a single PI-explanation for a DT.
\autoref{sec:exdt} proposes a solution for enumerating
PI-explanations of DTs, by reducing the problem to the computation of
MHSes.
%
\autoref{sec:res} studies the DTs obtained with two
state-of-the-art tools, on publicly available datasets, and shows that
paths in learned decision trees are often proper supersets of
PI-explanations. Moreover, the experimental results confirm that run
times for extracting PI-explanations are negligible.
%
Finally, \autoref{sec:conc} concludes the paper.

\section{Preliminaries} \label{sec:prelim}

\paragraph{Classification problems.}
We consider a classification problem defined on a set of features
$\fml{F}=\{1,\ldots,n\}$, where each feature $i$ takes values from a 
(categorical) domain $D_i$%
~\footnote{%
  For simplicity, most of the examples in the paper consider
  $D_i\triangleq\{0,1\}$ (i.e.\ binary features).
}, and $n$ denotes the number of features.
Feature space is defined by
$\mbb{F}=D_1\times{D_2}\times\ldots\times{D_n}$, each defining the
range of (categorical) values of each feature $x_i$. To refer to
an arbitrary point in feature space we use the notation
$\mbf{x}=(x_1,\ldots,x_n)$, whereas to refer to a concrete point in
feature space we use the notation $\mbf{v}=(v_1,\ldots,v_n)$, with
$v_i\in{D_i}$, $i=1,\ldots,n$.
We consider a binary classification problem, with two classes
$\fml{K}=\{\ominus,\oplus\}$%
~\footnote{%
  The results in the paper are readily applicable to multiple
  classes; the case $|\fml{K}|=2$ is considered solely for simplicity.}. 
(For simplicity, 
we will often use 0 for $\ominus$ and 1 for $\oplus$.)
An \emph{instance} (or example) denotes a pair $(\mbf{v},\pi)$, where 
$\mbf{v}\in\mbb{F}$ and $\pi\in\fml{K}$.
A machine learning model computes a function $\mu$ that maps the
feature space into the set of classes: $\mu:\mbb{F}\to\fml{K}$.
To train an ML model (in our case we are interested in DTs), we start
from a set of examples $\fml{E}=\{e_1,\ldots,e_m\}$, where each
$e_j=(\mbf{v}_j,\pi_j)$, such that $\mbf{v}_j\in\mbb{F}$  and
$\pi_j\in\fml{K}$, and $m$ is the number of examples. 

\paragraph{Decision trees.}
A decision tree $\fml{T}$ is a directed acyclic graph having at most
one path between every pair of nodes. $\fml{T}$ has a root node, 
characterized by having no incoming edges. All other nodes have one
incoming edge. We consider univariate decision trees (as opposed to
multivariate decision trees~\cite{utgoff-ml95}); %
hence, a non-terminal node is associated with a single feature $x_i$,
and each outgoing edge is associated with one (or more) values from
$D_i$.
Each terminal node is associated with a value of $\fml{K}$. An example
of a decision tree is shown in \autoref{fig:f02}. The number of
nodes in a DT is $r$.
When $|\fml{K}|=2$, a tree is characterized by two sets of paths,
where each path starts at the root and ends at a terminal node.
The set $\fml{P}=\{P_1,\ldots,P_{k_1}\}$ denotes the paths ending in a
$\oplus$ prediction.
The set $\fml{Q}=\{Q_1,\ldots,Q_{k_2}\}$ denotes the paths ending in a
$\ominus$ prediction. We will also use $\fml{R}=\fml{P}\cup\fml{Q}$.
A literal is of the form $x_i\Join{v_i}$, where
$\Join\:\in\{=,\not=\}$%
~\footnote{%
  The ideas described in the paper generalize to univariate DTs where
  features are either categorical or real- or integer-valued ordinal,
  and literals are of the form $x_i\Join{v_i}$, where
  $\Join\:\in\{<,\le,=,\not=,\ge,>\}$.
  The paper's main results can be extended to more general settings,
  but that these  are beyond the scope of the paper.%
}.
$x_i$ is a variable that denotes the value taken by feature $i$,
whereas $v_i\in{D_i}$ is a constant.
To model the operation of some DT learning tools~\cite{utgoff-ml97},
we allow generalized literals of the form $x_i\in{S_i}$, with
$S_i\subsetneq{D_i}$, such that the literal is consistent if the
feature is assigned a value in $S_i$.
Given this generalization, DTs correspond to multi-edge
decision trees~\cite{zeger-tit11}.
Moreover, two literals are inconsistent if they cause the feature to
take values that are inconsistent. For example, the literals $(x_1=0)$
and $(x_1=1)$ are inconsistent.
Each path in $\fml{T}$ is associated with a consistent conjunction of
literals, denoting the values assigned to the features so as to reach
the terminal node in the path. We will represent the set of literals
of some tree path by $\fml{L}(R_k)$, where $R_k$ is either a path in
$\fml{P}$ or $\fml{Q}$. Each path in the tree \emph{entails} the
prediction represented by path's terminal node. Let $\pi$ denote the 
prediction associated with path $R_k$. Then,
\begin{equation} \label{eq:ent01}
\forall(\mbf{x}\in\mbb{F}).
\left[
\bigwedge\nolimits_{(x_i\Join{v_i})\in\fml{L}(R_k)}(x_i\Join{v_i})
\right]
\limply(\mu(\mbf{x})=\pi)
\end{equation}
where $\mbf{x}=(x_1,\ldots,x_i,\ldots,x_n)$ and
$\pi\in\{\ominus,\oplus\}$.
%
Any pair of paths in $\fml{R}$ must have at least one pair
of inconsistent literals.

\paragraph{Interpretability \& DTs.}
Interpretability is generally regarded as a subjective concept,
without a rigorous definition~\cite{lipton-cacm18}, albeit different
authors have proposed different requirements for
interpretability~\cite{lipton-cacm18,flach-aaai19}.
Throughout this paper, we associate interpretability with irreducible
sets of feature-value pairs that are sufficient for the prediction%
\footnote{%
  Clearly, these subsets should be succinct, as it is generally
  accepted that human decision makers are only able to understand
  explanations with a reasonably small number of features.}.
Moreover, as argued in \autoref{sec:intro}, it is generally
accepted that DTs epitomize interpretability.  
Nevertheless, there is recent work that relates DT interpretability
with DTs being shallow~\cite{lipton-cacm18,muller-dsp18}, but also
work that proposes counterfactual explanations for meeting
interpretability desiderata for DTs~\cite{flach-aaai19}.

\paragraph{PI-explanations.}
The paper uses the definition of
PI-explanation~\cite{darwiche-ijcai18}, based on prime implicants of
some decision function.
Let us consider some ML model, computing a classification function
$\mu$ on feature space $\mbb{F}$, a point $\mbf{v}\in\mbb{F}$, with
prediction $\pi=\mu(\mbf{v})$, and let $E$ denote a set of literals
consistent with $\mbf{v}$ (and defined on features variables
$\mbf{x}$).
We say that $E$ is a PI-explanation for $\pi$ given $\mbf{v}$, if the
set of literals $E$ entails the prediction, and any proper subset of 
literals of $E$ does not entail the prediction. Formally, the
following conditions hold:
\begin{subequations}
  \begin{eqnarray}
    \forall(\mbf{x}\in\mbb{F}).
    \left[
    \bigwedge\nolimits_{l_i\in{E}}(l_i)
    \right]
    \limply(\mu(\mbf{x})=\pi)
    \\
    \forall(E'\subsetneq{E})
    \exists(\mbf{x}\in\mbb{F}).
    \left[
    \bigwedge\nolimits_{l_i\in{E'}}(l_i)
    \right]
    \land(\mu(\mbf{x})\not=\pi)
  \end{eqnarray}
\end{subequations}

Given a DT $\fml{T}$ and some path $R_k$, associated with some
prediction $\pi$, we say that path $R_k$ is
\emph{explanation-redundant} (or simply \emph{redundant}) if $R_k$ is
not a PI-explanation of $\pi$ given the ML model $\fml{T}$.
If we associate DT paths with instance explanations, then path
explanation-redundancy will manifest itself in instance explanations.
The concept of explanation-redundancy is illustrated
in~\autoref{ex:f02}.

\begin{figure}[t]
  \begin{subfigure}[b]{0.375\textwidth}
    \begin{center}
      \scalebox{0.9}{\forestset{
  BDT/.style={
    for tree={
      l=1.125cm,s sep=1.0cm,
      if n children=0{}{circle},
      draw,
      edge={
        my edge
      },
      if n=2 {
        edge+={0 my edge},
      }{},
    }
  },
}
\begin{forest}
  BDT
  [$x_1$
    [{\footnotesize\color{darkgreen}1}, edge label={node[near start,left,xshift=-2.75pt] {{\scriptsize1}}}]
    [$x_2$, edge label={node[near start,right,xshift=2.75pt] {{\scriptsize0}}}
     [{\footnotesize\color{darkgreen}1}, edge label={node[near start,left,xshift=-3.5pt] {{\scriptsize{1}}}}]
     [{\footnotesize\color{darkred}0}, edge label={node[near start,right,xshift=3.5pt] {{\scriptsize{0}}}}]
    ]
  ]
\end{forest}}
      \caption{DT for $f(x_1,x_2)={x_1}\lor{x_2}$} \label{fig:f01a}
    \end{center}
  \end{subfigure}
  \begin{subfigure}[b]{0.625\textwidth}
    \begin{center}
      \scalebox{0.9}{\forestset{
  BDT/.style={
    for tree={
      l=1.125cm,s sep=1.0cm,
      if n children=0{}{circle},
      draw,
      edge={
        my edge
      },
      if n=1{
        edge+={0 my edge},
      }{},
    }
  },
}
\begin{forest}
  BDT
  [$x_1$
    [$x_3$, edge label={node[near start,left,xshift=-5pt] {{\scriptsize0}}}
     [{\footnotesize\color{darkred}0}]
     [$x_4$
       [{\footnotesize\color{darkred}0}]
       [{\footnotesize\color{darkgreen}1}]
     ]
    ]
    [$x_2$, edge label={node[near start,right,xshift=5.75pt] {{\scriptsize1}}}
      [$x_3$
        [{\footnotesize\color{darkred}0}]
        [$x_4$
          [{\footnotesize\color{darkred}0}]
          [{\footnotesize\color{darkgreen}1}]
        ]
      ]
      [{\footnotesize\color{darkgreen}1}]
    ]
  ]
\end{forest}}
      \caption{DT for
        $f(x_1,\ldots,x_n)=\bigvee\nolimits_{i=1}^{n/2}x_{2i-1}\land{x_{2i}}$,
        with $n=4$} \label{fig:f01b}
    \end{center}
  \end{subfigure}
  \caption{Example DTs} \label{fig:f01}
\end{figure}
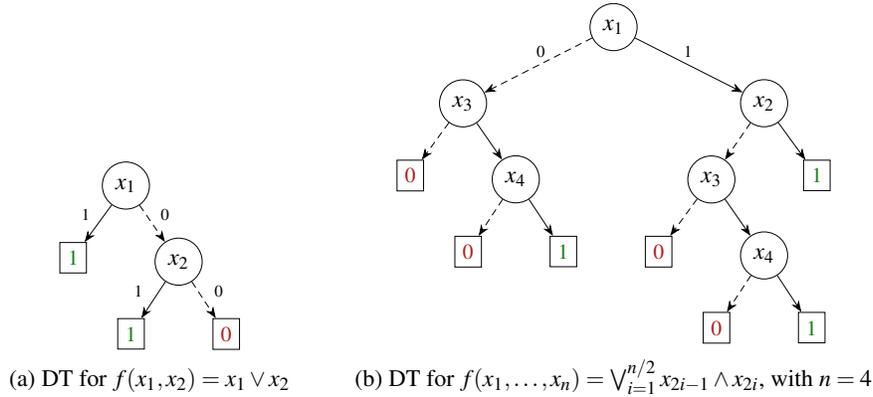

\begin{example} \label{ex:f02}
  Consider the DT in \autoref{fig:f01a}, for function
  $f(x_1,x_2)={x_1}\lor{x_2}$, and instance $(0,1)$. The path
  corresponds to the explanation $\{(x_1,0),(x_2,1)\}$ for prediction
  $f(0,1)=1$.
  Clearly, this path is explanation-redundant, as a PI-explanation for
  prediction 1 is $(x_2,1)$ (as is readily concluded from the function
  definition).\\
  As a less abstract example, we observe that the tree
  in~\autoref{fig:f01a} also models the example DT used
  in~\cite[Ch.~01,page~5]{zhou-bk12}%
  \footnote{This DT is shown in the supplementary material.},
  with $x_1$ denoting \emph{``is $y>0.73$?''}, $x_2$ denoting
  \emph{``is $x>0.64$?''}, class 1 denoting \emph{cross}, and class 0
  denoting \emph{circle}. Hence, a PI-explanation for the instance
  (\emph{no},\emph{yes}) with prediction ``cross'' is \emph{yes} to
  question \emph{``is $x>0.64$?''}, independently of the answer to
  question \emph{``is $y>0.73$?''}.
  \autoref{sec:rdt} investigates explanation-redundancy in greater
  detail.
\end{example}

\paragraph{Related work.}
As referenced in \autoref{sec:intro}, there exists a growing body
of work on exploiting DTs for interpretability.
To our best knowledge, the assessment of paths in DTs when compared to
PI-explanations has not been investigated. Recent
work~\cite{darwiche-corr20} outlines logical encodings of decision
trees, but that is orthogonal to the work reported in this paper.
In addition, there has been work on applying explainable AI (XAI) to
decision trees~\cite{lundberg-nature20}, but with the focus of
improving the quality of local (heuristic) explanations, where the
goal is to relate a local approximate model against a reference model;
hence there is no immediate relationship with
PI-explanations.

\section{Decision Trees May Not be Interpretable}
\label{sec:rdt}

This section shows that there exist Boolean functions for which a
learned decision tree will exhibit paths containing all features, and
for which a PI-explanation has a constant size. Thus, if we associate
explanations with DT paths, there will be explanations that are
arbitraly larger (on $n$) than the actual (constant-size)
PI-explanation.
As the experimental results demonstrate, and as discussed later in
this section, it is fairly frequent in practice for DT paths to
include more literals that those in the associated PI-explanations.
%

\begin{proposition} \label{prop:p01}
  There exist functions for which an irreducible DT contains paths
  which are a proper superset of a PI-explanation. Furthermore, the 
  difference in the number of literals is $n-k$, where $n$ is the
  number of features and $k$ is the (constant) size of a
  PI-explanation.
\end{proposition}

%
%
\begin{proof}
Let us consider the following Boolean function
$f:\{0,1\}^n\to\{0,1\}$ (with even $n$):
\begin{equation} \label{eq:f01}
  f(x_1,x_2,\ldots,x_{n-1},x_n)=\bigvee\nolimits_{i=1}^{n/2}x_{2i-1}\land{x_{2i}}
\end{equation}
For the case $n=4$, different off-the-self DT learners will obtain the
DT shown in~\autoref{fig:f01b}.
(To obtain the decision tree, we considered a dataset composed of
\emph{all} possible instances, and used
ITI~\cite{utgoff-ml97}\footnote{%
  ITI is available from~\url{https://www-lrn.cs.umass.edu/iti/}. 
We considered a number of publicly available DT learners, and reached
the same conclusions (in terms of path length) in all cases.}.)
%
Furthermore, it is immediate to conclude that the decision tree shown
is irreducible (i.e.\ no nodes can be removed while keeping accuracy).
Moreover, let the target instance be $(\mbf{v},\pi)=((1,0,1,1),1)$.
In this case, the explanation (i.e.\ the path) extracted from the DT
is $(x_1=1)\land(x_2=0)\land(x_3=1)\land(x_4=1)$, which guarantees
that the prediction is 1.
%
However, it is immediate from the function definition \eqref{eq:f01},
that $(x_3=1)\land(x_4=1)$ entails $f(x_1,x_2,x_3,x_4)=1$,
\emph{independently} of the value assigned to $x_1$ and $x_2$,
i.e.\ in this case the PI-explanation is $(x_3=1)\land(x_4=1)$.
The same analysis generalizes to an arbitrary number of variables.
For an instance of the form $((1,0,1,0,\ldots,1,0,1,1),1)$, the DT
would indicate an explanation with $n$ literals, whereas the
PI-explanation has size 2, namely $(x_{n-1}=1)\land(x_n=1)$.
\end{proof}

It should be noted that the issue above does not depend on whether the
DT is redundant (e.g.\ in the cases shown, the DTs are \emph{not}
redundant); the reported issues result solely from a fundamental
limitation of DTs for succinctly representing certain classes of
functions%
\footnote{%
  Indeed, it is well-known that DTs are not as succinct as decision
  lists (DLs)~\cite{rivest-ml87} (and so not as succinct as decision
  sets (DSs)).
  This means that there exist functions that have succint DLs or DSs,
  but not DTs.
  Although not the focus of the paper,
  we conjecture that similar results can be obtained for DLs and for
  restricted cases of DSs, among those where DSs compute functions.
}.
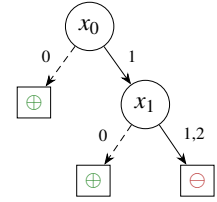
\begin{figure}[t]
  \begin{subfigure}[b]{0.75\textwidth}
    \begin{center}
      \renewcommand{\arraystretch}{1.05}
      \scalebox{0.95}{%
        \begin{tabular}{ccccc} \toprule
          Feature & Name & Domain & Values & Meaning \\ \midrule
          $x_0$   & Humidity & $D_0$ & $\{0,1\}$ & normal,high\\
          $x_1$   & Outlook & $D_1$ & $\{0,1,2\}$ & overcast,rain,sunny\\
          $x_2$   & Wind & $D_2$ & $\{0,1\}$ & strong,weak\\ \bottomrule
        \end{tabular}
      }
    \end{center}
    \caption{Features for \emph{PlayTennis} dataset} \label{fig:f02a}
  \end{subfigure}
  \begin{subfigure}[b]{0.2125\textwidth}
    \begin{center}
      \scalebox{0.925}{\forestset{
  BDT/.style={
    for tree={
      l=1.125cm,s sep=1.0cm,
      if n children=0{}{circle},
      draw,
      edge={
        my edge
      },
      if n=1{
        edge+={0 my edge},
      }{},
    }
  },
}
\begin{forest}
  BDT
  [$x_0$
    [{\footnotesize\color{darkgreen}$\oplus$}, edge label={node[near start,left,xshift=-3.5pt] {{\scriptsize{0}}}}]
    [$x_1$, edge label={node[near start,right,xshift=3.5pt] {{\scriptsize{1}}}}
      [{\footnotesize\color{darkgreen}$\oplus$}, edge label={node[near start,left,xshift=-3.5pt] {{\scriptsize{0}}}}]
      [{\footnotesize\color{darkred}$\ominus$}, edge label={node[near start,right,xshift=3.5pt] {{\scriptsize{1,2}}}}]
    ]
  ]
\end{forest}
}
    \end{center}
    \caption{Resulting DT} \label{fig:f02b}
  \end{subfigure}
  \caption{Another Example DT} \label{fig:f02}
\end{figure}
\autoref{fig:f02} exemplifies that redundancy may occur even in very
simple DTs.
%
The DT was obtained with the optimal decision tree package from %
Interpretable~AI%
~\cite{bertsimas-ml17,iai}\footnote{%
  We used the well-known \emph{PlayTennis} dataset%
  ~\cite{mitchell-bk97}.}, where the $\ominus$ and the leftmost
$\oplus$ are predicted with 75\% and 83.3\% confidence, respectively.
(The branch annotated with $1,2$ denotes that $x_1\in\{1,2\}$.)
%
Let the instance be
$(\tn{Humidity},\tn{Outlook},\tn{Wind})=(\tn{high},\tn{overcast},\tn{weak})$. As
an explanation we could use the literals in the tree path,
i.e.\ $\{(\tn{Humidity}=\tn{high}),(\tn{Outlook}=\tn{overcast})\}$.
However, careful analysis allows us to conclude that
$\{(\tn{Outlook}=\tn{overcast})\}$ suffices to entail the prediction
$\oplus$, i.e.\ as long as `Outlook' is `overcast', the prediction
will be $\oplus$ \emph{independently} of the value of `Humidity'.

\section{Extracting PI-Explanations from DTs} \label{sec:xdt}

\paragraph{Deciding explanation-redundancy with NP oracles.}
%
%
Let us consider a decision tree $\fml{T}$, with sets of paths
$\fml{P}$ and $\fml{Q}$, denoting respectively the paths with
prediction $\oplus$ and $\ominus$. Let us also consider an instance
$(\mbf{v},\oplus)$, with $\mbf{v}\in\mbb{F}$ and $\oplus\in\fml{K}$
(the case for $\ominus$ would be similar), and let $P_k$ denote that
path consistent with $\mbf{v}$.
To decide whether $P_k$ exhibits explanation-redundancy, one possible
solution is to use an NP oracle.

For an instance consistent with $P_k$, we can model the prediction of
the decision tree (for prediction $\oplus$) as follows:
\begin{equation} \label{eq:entail01}
  \bigwedge\nolimits_{l_j\in{\fml{L}(P_k)}}(l_j)\entails\bigvee\nolimits_{P_i\in\fml{P}}\bigwedge\nolimits_{l_s\in{\fml{L}(P_i)}}(l_s)
\end{equation}
Hence, we require the unsatisfiability of,
\begin{equation} \label{eq:entail02}
  \bigwedge\nolimits_{l_j\in{\fml{L}(P_k)}}(l_j)\land\bigwedge\nolimits_{P_i\in\fml{P}}\bigvee\nolimits_{l_s\in{\fml{L}(P_i)}}(\neg{l_s})
\end{equation}
Now, if there exists a literal $l_j$ that can be dropped from $P_k$
such that unsatisfiability is preserved, then $P_k$ exhibits
explanation-redundancy. Hence, we need at most $m$ calls to an NP
oracle to decide explanation-redundancy, where $m$ is the number of
features. (This high-level procedure was proposed in earlier work in
more general terms~\cite{inms-aaai19}.)
However, the special structure of a DT, makes the problem far
simpler, and can be solved in polynomial time, as shown next.

\paragraph{Deciding explanation-redundancy in linear time.}
%
Observe that~\eqref{eq:entail01} is preserved iff 
at least one of the features with a literal in $P_k$ has another
literal that is false along \emph{any} path that yields prediction
$\ominus$.
Thus, there is explanation-redundancy if there exists a literal from
$P_k$ that can be dropped while the remaining literals in $P_k$ still
guarantee that at least one literal is false along any path in
$\fml{Q}$.
Before detailing a polynomial time algorithm for deciding
explanation-redundancy, let us consider a concrete example.

\begin{example} \label{ex:02}
  For the DT from \autoref{fig:f01b} we have,
  \[
  \begin{array}{l}
    \fml{P}=\{P_1,P_2,P_3\} \\[1.5pt]  
    \fml{L}(P_1)=\{(x_1=0),(x_3=1),(x_4=1)\} \\
    \fml{L}(P_2)=\{(x_1=1),(x_2=0),(x_3=1),(x_4=1)\} \\
    \fml{L}(P_3)=\{(x_1=1),(x_2=1)\} \\[3.5pt]
    \fml{Q}=\{Q_1,Q_2,Q_3,Q_4\} \\[1.5pt]  
    \fml{L}(Q_1)=\{(x_1=0),(x_3=0)\} \\
    \fml{L}(Q_2)=\{(x_1=0),(x_3=1),(x_4=0)\} \\
    \fml{L}(Q_3)=\{(x_1=1),(x_2=0),(x_3=0)\} \\
    \fml{L}(Q_4)=\{(x_1=1),(x_2=0),(x_3=1),(x_4=0)\} \\
  \end{array}
  \]
  We consider path $P_2$ (and so any feature space point consistent
  with $P_2$).
  We can readily conclude that if literal $(x_1=1)$ is removed from
  the literals of $P_2$, all the paths in $\fml{Q}$ remain
  inconsistent. This is true for example  
  because $(x_3=1)$ and $(x_4=1)$. Similarly, we could drop literal
  $(x_2=0)$.
\end{example}

The example above naturally suggests a quadratic time (on the size $n$
of the decision tree) algorithm to decide whether a tree path exhibits 
explanation-redundancy. Concretely, for each literal in $P_k$, we
analyze each path in $\fml{Q}$ whether it is still inconsistent. If
all paths in $\fml{Q}$ remain inconsistent, then the literal can be
dropped, and the path exhibits explanation-redundancy.
Nevertheless, it is possible to devise a more efficient solution, one
that runs in linear time.

The proposed algorithm analyzes the features containing literals in
$P_i$\footnote{%
  As before, we assume a path $P_i$ with prediction $\oplus$.
},
in turn allowing each to be declared \emph{universal},
i.e.\ the feature can take any value. For each feature $f_j$, the
algorithm recursively analyzes paths with a different prediction,
checking whether each such path is inconsistent, due to some other
literal. If that is the case, the path $P_i$ is explanation-redundant,
at least due to feature $f_j$.
The algorithm analyzes the internal nodes of $P_i$ in reverse order,
starting from deepest non-terminal node in the path. (For now, we
assume that $P_i$ has at most one literal on any given feature; this
restriction will be lifted below.)
For each non-terminal node $p_j\in{P}_i$, the associated feature $f_j$
is made universal, i.e.\ it can take 
\emph{any} value. Starting at $p_j$, all child nodes not in $P_i$ are
recursively visited, checking for a consistent sub-path (starting at
$p_j$) to a terminal node with prediction $\ominus$. If such path
exists, then $f_j$ cannot be made universal, and so it cannot be
discarded from a PI-explanation. Otherwise, all sub-paths to
prediction $\ominus$ are inconsistent, and so the value of $f_j$ is
irrelevant for the prediction. A path is declared redundant iff at
least one feature is declared redundant. With filtering
(i.e.\ $\mathtt{rec}=0$), each tree node is analyzed at most once, and
so the amortized run time of the algorithm over all features is
$\fml{O}(|\fml{T}|)$. 
If a feature is tested more than once along $P_i$, the algorithm
requires minor modifications. In this case, a decision about whether
feature $f_i$ can take any value, can only be made once all nodes
involving $f_i$ have been analyzed and, for all such nodes, the
feature $f_i$ has been declared redundant.
\begin{algorithm}[t]
%
\SetKwFunction{ChkPaths}{\small{\sc DecidePathRedundancy}}
\SetKwFunction{ChkDown}{\small{\sc ChkDown}}
\SetKwFunction{AggrFeat}{AggrFeatureNodes}
\SetKwFunction{PathFeats}{PathFeatures}
\SetKwFunction{PathNodes}{PathNodes}
\SetKwFunction{SetUniv}{SetUniversal}
\SetKwFunction{UnsetUniv}{UnsetUniversal}
\SetKwFunction{ReportPath}{ReportPath}

\Func \ChkPaths{$\fml{T}$} \\ 
\Indp
{
  \lnlset{d1:1}{1}
  \ForEach{$R_k\in\fml{T}$}{
    \lnlset{d1:2}{2}
    $\fml{A}\gets\AggrFeat(\fml{T},R_k)$ \;
    \lnlset{d1:3}{3}
    $\fml{N}\gets\PathNodes(\fml{T},R_k)$ \;
    \lnlset{d1:4}{4}
    $\tn{isPathRed}\gets\false$ \;
    \lnlset{d1:5}{5}
    \ForEach{$f\in\PathFeats(\fml{T},R_k)$}{
      \lnlset{d1:6}{6}
      $\tn{isFeatRed}\gets\true$ \;
      \lnlset{d1:7}{7}
      $\SetUniv(\fml{T},f)$ \;
      \lnlset{d1:8}{8}
      \ForEach{$n\in\fml{A}(f)$}{
        \lnlset{d1:9}{9}
        \If{\KwNot $\ChkDown(\fml{T},\fml{N},R_k,n,0)$}{
          \lnlset{d1:10}{10}
          $\tn{isFeatRed}\gets\false$ \;
          \lnlset{d1:11}{11}
          \Break \;
        }
      }
      \lnlset{d1:12}{12}
      $\UnsetUniv(\fml{T},f)$ \;
      \lnlset{d1:13}{13}
      \If{$\tn{isFeatRed}$}{
        \lnlset{d1:14}{14}
        $\tn{isPathRed}\gets\true$ \;
        \lnlset{d1:15}{15}
        \Break \;
      }
    }
    \lnlset{d1:16}{16}
    $\ReportPath(R_k,\tn{isPathRed})$ \;
  }
}
\Indm
\BotBlankLine
%

  \caption{Deciding path redundancy} \label{alg:pathred}
\end{algorithm}
\begin{algorithm}[t]
%
\SetKwFunction{ChkDown}{\small{\sc ChkDown}} %
\SetKwFunction{CDCall}{ChkDown} %
\SetKwFunction{ChildNodes}{ChildNodes}
\SetKwFunction{Term}{IsTerminal}
\SetKwFunction{Pred}{Prediction}
\SetKwFunction{HasP}{HasPaths}
\SetKwFunction{Univ}{Universal}
\SetKwFunction{Feat}{Feature}

\Func \ChkDown{$\fml{T},\fml{N},R_k,n,\mathtt{rec}$} \\ 
\Indp
{
  \lnlset{c1:1}{1}
  $\pi\gets\Pred(\fml{T},R_k)$ \;
  \lnlset{c1:2}{2}
  $\fml{C}\gets\ChildNodes(\fml{T},n)$ \;
  \lnlset{c1:3}{3}
  \ForEach{$c\in\fml{C}$}{
    \lnlset{c1:4}{4}
    \lIf{$c\in\fml{N}$ \KwAnd $\mathtt{rec}=0$}{\Cont}
    \lnlset{c1:5}{5}
    \lIf{\KwNot $\HasP(\fml{T}, c, \pi)$}{\Cont}
    \lnlset{c1:6}{6}
    \If{$\Term(\fml{T},c)$}{
      \lnlset{c1:7}{7}
      \Return{\false} \;
    }
    \lnlset{c1:8}{8}
    $g\gets\Feat(\fml{T},c)$ \;
    \lnlset{c1:9}{9}
    \lIf{\KwNot $\Univ(\fml{T},g)$}{\Cont}
    \lnlset{c1:10}{10}
    \If{\KwNot $\ChkDown(\fml{T},\fml{N},R_k,c,\mathtt{rec})$} {
      \lnlset{c1:11}{11}
      \Return{\false} \;
    }
  }
  \lnlset{c1:12}{12}
  \Return{\true}\;
}
\Indm
\BotBlankLine
%

  \caption{Inconsistent sub-path lookup} \label{alg:chkdown}
\end{algorithm}
Algorithms~\ref{alg:pathred} and~\ref{alg:chkdown} summarize the two
main steps of the proposed algorithm. The auxiliar functions serve to
test/set whether a feature is universal (i.e.\ whether it can take any
value of its domain) (resp.~\Univ/\SetUniv), aggregate the nodes
associated with a given feature along some path (\AggrFeat), list the
nodes in a given path (\PathNodes), get the  prediction of some path
(\Pred), list the child nodes of some node in the tree (\ChildNodes),
get the feature associated with a node (\Feat), check whether some
node is terminal (\Term) and, finally, whether the sub-paths starting
at some node can reach a prediction other than the one associated with
the target path $R_k$ (\HasP).
Finally, the argument $\mathtt{rec}$ of $\ChkDown$ is a flag that
serves to avoid re-visiting already visited paths. For removing
redundant features this filtering does not apply, as clarified below.

\begin{example}
  We consider again the DT from \autoref{fig:f01b} and path
  $P_2$, where $\fml{L}(P_2)=\{(x_1=1),(x_2=0),(x_3=1),(x_4=1)\}$, and
  prediction $1$. The nodes of $P_2$ are analyzed in the order
  $\langle{x_4},{x_3},{x_2},{x_1}\rangle$.
  Clearly, the sub-path consistent with $(x_4=0)$ yields prediction
  $0$. Hence, feature $x_4$ (associated with literal $x_4=1$) is not
  redundant.
  For $x_3$, we consider (\emph{only}) the sub-path corresponding to
  $(x_3=0)$. Again, the prediction is $0$, and so the feature $x_3$ is
  not redundant. For feature $x_2$, the only sub-path corresponds to
  $(x_2=1)$, for which the prediction remains unchanged. Hence, $x_2$
  can take \emph{any} value, and so it is declared redundant. As a
  result, $P_2$ is also declared redundant.
  It is helpful to analyze the execution of the algorithm for $x_1$.
  The sub-paths consistent with $(x_1=1)$ correspond to $P_1$, $Q_1$
  and $Q_2$.
  Hence, due to $Q_1$ and $Q_2$, it might seem that $x_1$ might be
  irredundant. However, both $Q_1$ and $Q_2$ are inconsistent with
  other non-redundant literals of $P_2$, concretely $x_3=1$ and
  $x_4=1$. Hence, $x_1$ is declared redundant.
\end{example}

\paragraph{Extracting one PI-explanation.}
%
%
One approach to find a PI-explanation is to use recent work based on
compilation~\cite{darwiche-ijcai18} or iterative entailment
checks with an NP oracle~\cite{inms-aaai19}.
However, a computationally simpler solution is based on the ideas
described above. 

Features are analyzed as proposed in Algorithms~\ref{alg:pathred}
and~\ref{alg:chkdown}. However, a feature already declared as
redundant signifies that it can take \emph{any} value from its
domain. This may allow inconsistent paths with prediction $\ominus$ to
become consistent, thus preventing some other feature from being
declared redundant.
One consequence is that the filtering of paths exploited in
Algorithms~\ref{alg:pathred} and~\ref{alg:chkdown} is not longer
applicable. Concretely, the analysis of each feature requires visiting
all of the sub-paths starting in any of the path nodes testing the
feature. Algorithm~\ref{alg:chkdown} can still be used, in this case
by setting $\mathtt{rec}$ to 1, i.e.\ all sub-paths will be analyzed.
The worst-case complexity of the resulting algorithm is
$\fml{O}(|\fml{T}|)$ for each feature, thus yielding
$\fml{O}(|\fml{T}|^2)$ for the complete algorithm. Finally, we can
conclude that an algorithm for finding one PI-explanation for each
path in $\fml{T}$ runs in $\fml{O}(|\fml{T}|^3)$.

\begin{example}
  We consider the DT from \autoref{fig:f01b} and analyze path
  $P_2$ (and so \emph{any} instance consistent with $P_2$). Literals
  are analyzed in reverse path order.
  (In this case, aggregation of nodes by feature is optional, as long
  as a decision with respect to a feature is delayed until all nodes
  associated with the feature have been analyzed, and keeping track
  whether non-irredundancy applies to all nodes.)
  As before, the literal $(x_4=1)$ is not redundant; otherwise $Q_4$
  would be consistent. Similarly, the literal $(x_3=1)$ is not
  redundant; otherwise $Q_3$ would be consistent.
  In contrast, the feature $x_2$ can be made universal, as this does
  not change the prediction, i.e.\ $P_2$ is consistent with the
  prediction, and the other literals $(x_3=1)$ and $(x_4=1)$ block the
  paths in $\fml{Q}$.
  A similar analysis applies in the case of feature $x_1$. In this
  case, the algorithm analyzes all paths (since 
  Algorithm~\ref{alg:chkdown} is invoked with $\mathtt{rec}=1$).
  Due to the literals $(x_3=1)$ and $(x_4=1)$, all paths in $\fml{Q}$
  are inconsistent. As a result, feature $x_1$ can be declared
  redundant,
  and a PI-explanation for $P_2$ is thus $\{(x_3=1),(x_4=1)\}$.
\end{example}

\paragraph{Path-restricted vs.\ path-unrestricted explanations.}
%
%
For the algorithms described earlier
we also need to decide the set of literals to consider. 
One option is the set of literals specified by the instance.
Another option is the set of literals specific to the tree path
consistent with the instance.
If we are interested in finding PI-explanations for the prediction,
given the instance, then we should consider the literals specified by
the instance. However, if we want to report explanations that relate
with the tree path consistent with the instance, then we should
consider only the literals in the tree path.
Clearly, path-restricted PI-explanations are a subset of
path-unrestricted PI-explanations.
The following example illustrates the differences between the two
approaches.

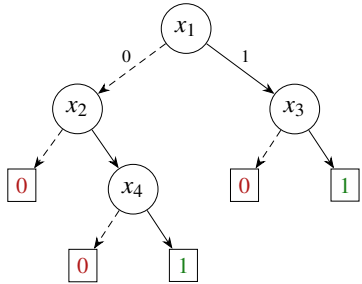
\begin{figure}[t]
  \centering \scalebox{0.95}{\forestset{
  BDT/.style={
    for tree={
      l=1.125cm,s sep=1.0cm,
      if n children=0{}{circle},
      draw,
      edge={
        my edge
      },
      if n=1{
        edge+={0 my edge},
      }{},
    }
  },
}
\begin{forest}
  BDT
  [$x_1$
    [$x_2$, edge label={node[near start,left,xshift=-3.5pt] {{\scriptsize0}}}
      [{\footnotesize\color{darkred}0}]
      [$x_4$
        [{\footnotesize\color{darkred}0}]
        [{\footnotesize\color{darkgreen}1}]
      ]
    ]
    [$x_3$, edge label={node[near start,right,xshift=3.5pt] {{\scriptsize1}}}
     [{\footnotesize\color{darkred}0}]
     [{\footnotesize\color{darkgreen}1}]
    ]
  ]
\end{forest}}
  \caption{Path-restricted vs.\ path-unrestricted explanations}
  \label{fig:f03} 
\end{figure}

\begin{example} \label{ex:f03}
  We consider the example DT shown in~\autoref{fig:f03}.
  Given the point $(x_1,x_2,x_3,\linebreak[1]x_4)=(1,1,1,1)$, the consistent path
  in the DT consists of the literals $\{(x_1=1),(x_3=1)\}$.
  The only PI-explanation that is restricted to this path is the path
  itself: $\{(x_1=1),(x_3=1)\}$. However, if we enable
  path-unrestricted explanations, we will also obtain the
  PI-explanation $\{(x_2=1),(x_3=1),(x_4=1)\}$, which is not even
  shown as (part of) a path in the decision tree.
\end{example}

\section{Enumeration of PI-Explanations} \label{sec:exdt} 

The enumeration of multiple (or all) PI-explanations can help
human decision makers to develop a better understanding of some
prediction, but also of the underlying ML model.
%
Recent work~\cite{darwiche-ijcai18} compiles a decision function into a
Sentential Decision Diagram (SDD), from which the enumeration of
PI-explanations can be instrumented. Moreover, from a compiled
representation of the PI-explanations, each PI-explanation can be
reported in polynomial time. The downside is that these representation
are worst-case exponential in the size of the original ML model.
Another line of work for computing PI-explanations is based on
iterative entailment checks using an NP-oracle~\cite{inms-aaai19}.
However, this recent work does not address the enumeration of
PI-explanations.
This section develops a solution for the enumeration of
PI-explanations in the case of DTs, by reduction to the enumeration of
minimal hitting sets (MHSes).

We consider the situation where the prediction is $\oplus$ for some
point $\mbf{v}$ in feature space, which is consistent with some tree
path $P_k\in\fml{P}$. As a result, each path in $\fml{Q}$ is
inconsistent with at least one literal, among those either associated
with $\mbf{v}$ or with $P_k$. Let the set of literals considered be
$R$.
For each path $Q_s\in\fml{Q}$, let $L_s$ denote the
set of literals in $R$ that are inconsistent with $Q_s$. For a subset
$S$ of $R$ to entail the prediction, it must hit each set of literals
$L_s$. Among the possible sets $S$, each subset-minimal set is a
PI-explanation.
Thus, we can list the PI-explanations (starting from the literals
taken from $\mbf{v}$ or from $P_k$) by enumerating minimal hitting
sets.

\begin{example}
  For $P_2$ in the DT of Figure~\ref{fig:f01}, the sets to hit are:
  \[
  \begin{array}{lcl}
    Q_1:~\{(x_1=1),(x_3=1)\} & \quad &
    Q_2:~\{(x_1=1),(x_4=1)\} \\[1.25pt]
    Q_3:~\{(x_3=1)\} & \quad &
    Q_4:~\{(x_4=1)\}\\
  \end{array}
  \]
  In this case, the only MHS is $\{(x_3=1),(x_4=1)\}$, representing a
  single PI-explanation $\{(x_3=1),(x_4=1)\}$.
\end{example}

\begin{example}
  For the DT in Figure~\ref{fig:f03}, and by considering the
  literals from the point $\mbf{v}=(1,1,1,1)$, 
  the sets to hit are then:
  \[
  \begin{array}{lcl}
    Q_1:~\{(x_1=1),(x_2=1)\} & \quad & 
    Q_2:~\{(x_1=1),(x_4=1)\} \\[1.25pt]
    Q_3:~\{(x_3=1)\} & \quad & \\
  \end{array}
  \]
  The MHSes are $\{(x_1=1),(x_3=1)\}$ and
  $\{(x_2=1),(x_4=1),(x_3=1)\}$, each denoting a path-unrestricted
  PI-explanation.
\end{example}

\section{Experimental Results} \label{sec:res}

\setlength{\tabcolsep}{5pt}
\rowcolors{2}{gray!10}{}
\begin{table*}[h]\centering

\resizebox{\textwidth}{!}{
  \begin{tabular}{l>{\lpr}S[table-format=3.0,table-space-text-pre=\lpr]S[table-format=5.0,table-space-text-post=\rpr]<{\rpr}cS[table-format=3.0]S[table-format=3.0]S[table-format=2.0]S[table-format=2.0]S[table-format=2.0]cccS[table-format=2.0]S[table-format=4.0]S[table-format=3.0]S[table-format=4.0]S[table-format=2.0]S[table-format=2.0]S[table-format=2.0]cS[table-format=2.0]}
\toprule[1.2pt]
\rowcolor{white}
\multirow{2}{*}{\bf Dataset} & \multicolumn{2}{c}{\multirow{2}{*}{\bf (\#F\hspace{0.5cm} \#S)}} & \multicolumn{9}{c}{\bf IAI} & \multicolumn{9}{c}{\bf ITI} \\
  \cmidrule[0.8pt](lr{.75em}){4-12}
  \cmidrule[0.8pt](lr{.75em}){13-21}
\rowcolor{white}
& \multicolumn{2}{c}{} & {\bf D} & {\bf \#N} & {\bf \%A} & {\bf \#P} & {\bf \%R} & {\bf \%C} & {\bf \%m} & {\bf \%M} & {\bf \%avg} & {\bf D} & {\bf \#N} & {\bf \%A} & {\bf \#P} & {\bf \%R} & {\bf \%C} & {\bf \%m} & {\bf \%M} & {\bf \%avg} \\
\toprule[1.2pt]

adult & 12 & 6061 & 6 & 83 & 78 & 42 & 33 & 25 & 20 & 40 & 25 & 17 & 509 &  73 & 255 & 75 & 91 & 10 & 66 & 22  \\

anneal & 38 & 886 & 6 & 29 & 99 & 15 & 26 & 16 & 16 & 33 & 21 & 9 & 31 &  100 & 16 & 25 & 4 & 12 & 20 & 16  \\

backache & 32 & 180 & 4 & 17 & 72 & 9 & 33 & 39 & 25 & 33 & 30 & 3 & 9 &  91 & 5 & 80 & 87 & 50 & 66 & 54  \\ 

bank & 19 & 36293 & 6 & 113 & 88 & 57 & 5 & 12 & 16 & 20 & 18 & 19 & 1467 &  86 & 734 & 69 & 64 & 7 & 63 & 27  \\

biodegradation & 41 & 1052 & 5 & 19 & 65 & 10 & 30 & 1 & 25 & 50 & 33 & 8 & 71 &  76 & 36 & 50 & 8 & 14 & 40 & 21  \\

cancer & 9 & 449 & 6 & 37 & 87 & 19 & 36 & 9 & 20 & 25 & 21 & 5 & 21 &  84 & 11 & 54 & 10 & 25 & 50 & 37  \\

car & 6 & 1728 & 6 & 43 & 96 & 22 & 86 & 89 & 20 & 80 & 45 & 11 & 57 &  98 & 29 & 65 & 41 & 16 & 50 & 30  \\

colic & 22 & 357 & 6 & 55 & 81 & 28 & 46 & 6 & 16 & 33 & 20 & 4 & 17 &  80 & 9 & 33 & 27 & 25 & 25 & 25  \\

compas & 11 & 1155 & 6 & 77 & 34 & 39 & 17 & 8 & 16 & 20 & 17 & 15 & 183 &  37 & 92 & 66 & 43 & 12 & 60 & 27  \\

contraceptive & 9 & 1425 & 6 & 99 & 49 & 50 & 8 & 2 & 20 & 60 & 37 & 17 & 385 &  48 & 193 & 27 & 32 & 12 & 66 & 21  \\

dermatology & 34 & 366 & 6 & 33 & 90 & 17 & 23 & 3 & 16 & 33 & 21 & 7 & 17 &  95 & 9 & 22 & 0 & 14 & 20 & 17  \\ 

divorce & 54 & 150 & 5 & 15 & 90 & 8 & 50 & 19 & 20 & 33 & 24 & 2 & 5 &  96 & 3 & 33 & 16 & 50 & 50 & 50  \\

german & 21 & 1000 & 6 & 25 & 61 & 13 & 38 & 10 & 20 & 40 & 29 & 10 & 99 &  72 & 50 & 46 & 13 & 12 & 40 & 22  \\

heart-c & 13 & 302 & 6 & 43 & 65 & 22 & 36 & 18 & 20 & 33 & 22 & 4 & 15 &  75 & 8 & 87 & 81 & 25 & 50 & 34  \\

heart-h & 13 & 293 & 6 & 37 & 59 & 19 & 31 & 4 & 20 & 40 & 24 & 8 & 25 &  77 & 13 & 61 & 60 & 20 & 50 & 32  \\

kr-vs-kp & 36 & 3196 & 6 & 49 & 96 & 25 & 80 & 75 & 16 & 60 & 33 & 13 & 67 &  99 & 34 & 79 & 43 & 7 & 70 & 35  \\

lending & 9 & 5082 & 6 & 45 & 73 & 23 & 73 & 80 & 16 & 50 & 25 & 14 & 507 &  65 & 254 & 69 & 80 & 12 & 75 & 25  \\

letter & 16 & 18668 & 6 & 127 & 58 & 64 & 1 & 0 & 20 & 20 & 20 & 46 & 4857 &  68 & 2429 & 6 & 7 & 6 & 25 & 9  \\

lymphography & 18 & 148 & 6 & 61 & 76 & 31 & 35 & 25 & 16 & 33 & 21 & 6 & 21 &  86 & 11 & 9 & 0 & 16 & 16 & 16  \\

mortality & 118 & 13442 & 6 & 111 & 74 & 56 & 8 & 14 & 16 & 20 & 17 & 26 & 865 &  76 & 433 & 61 & 61 & 7 & 54 & 19  \\ 

mushroom & 22 & 8124 & 6 & 39 & 100 & 20 & 80 & 44 & 16 & 33 & 24 & 5 & 23 &  100 & 12 & 50 & 31 & 20 & 40 & 25  \\

pendigits & 16 & 10992 & 6 & 121 & 88 & 61 & 0 & 0 & \textemdash & \textemdash & \textemdash & 38 & 937 &  85 & 469 & 25 & 86 & 6 & 25 & 11  \\

promoters & 58 & 106 & 1 & 3 & 90 & 2 & 0 & 0 & \textemdash & \textemdash & \textemdash & 3 & 9 &  81 & 5 & 20 & 14 & 33 & 33 & 33  \\

recidivism & 15 & 3998 & 6 & 105 & 61 & 53 & 28 & 22 & 16 & 33 & 18 & 15 & 611 &  51 & 306 & 53 & 38 & 9 & 44 & 16  \\

seismic\_bumps & 18 & 2578 & 6 & 37 & 89 & 19 & 42 & 19 & 20 & 33 & 24 & 8 & 39 &  93 & 20 & 60 & 79 & 20 & 60 & 42  \\

shuttle & 9 & 58000 & 6 & 63 & 99 & 32 & 28 & 7 & 20 & 33 & 23 & 23 & 159 &  99 & 80 & 33 & 9 & 14 & 50 & 30  \\

soybean & 35 & 623 & 6 & 63 & 88 & 32 & 9 & 5 & 25 & 25 & 25 & 16 & 71 &  89 & 36 & 22 & 1 & 9 & 12 & 10  \\


spambase & 57 & 4210 & 6 & 63 & 75 & 32 & 37 & 12 & 16 & 33 & 19 & 15 & 143 &  91 & 72 & 76 & 98 & 7 & 58 & 25  \\

spect & 22 & 228 & 6 & 45 & 82 & 23 & 60 & 51 & 20 & 50 & 35 & 6 & 15 &  86 & 8 & 87 & 98 & 50 & 83 & 65  \\

splice & 2 & 3178 & 3 & 7 & 50 & 4 & 0 & 0 & \textemdash & \textemdash & \textemdash & 88 & 177 &  55 & 89 & 0 & 0 & \textemdash & \textemdash & \textemdash  \\


\bottomrule[1.2pt]
\end{tabular}
}
\caption{Explanation-redundancy  in decision trees obtained with IAI and ITI. \label{tab:res}}
\end{table*}

This section presents a summary of experimental evaluation of the
explanation-redundancy of two state-of-the-art heuristic DT
classifiers.
Concretely, we use the well-known DT training tools \emph{ITI}
(\emph{Incremental Tree Induction})~\cite{utgoff-ml97,iti} and
\emph{IAI} (\emph{Interpretable AI})~\cite{bertsimas-ml17,iai}.
ITI is run with the pruning option enabled, which helps avoiding
overfitting and aims at constructing shallow DTs.
To enforce IAI to produce shallow DTs and achieve high accuracy, it is
set to use the optimal tree classifier method with the maximal depth
of 6%
\footnote{Our results confirm that larger maximal depths would in most
cases increase the percentage of redundant paths. A smaller maximal
depth would not improve accuracy.}.
The experiments consider datasets with categorical (non-binarized)
data, which both ITI and IAI can handle%
%
\footnote{%
Other known DT learning tools, including
scikit-learn~\cite{scikitlearn-full} and
DL8.5~\cite{schaus-aaai20,schaus-ijcai20a} can only handle numerical 
and binary features, respectively, and so could not be included in the
experiments.}.
The assessment is performed on a selection of 80 publicly
available datasets, which originate from \emph{UCI Machine Learning
  Repository}~\cite{uci}, \emph{Penn Machine Learning
  Benchmarks}~\cite{pennml}, and \emph{OpenML
repository}~\cite{openml}.
(Due to space restrictions, we report the results only for 30 datasets
but the results shown extend to the complete benchmark set, and are
included in the supplementary materials.)
The number of features (data instances, resp.) in the benchmark suite
vary from 2 to 118 (106 to 58000, resp.) with the average being 31.2
(6045.3, resp.).


The experiments are performed on a MacBook Pro with a Dual-Core Intel
Core~i5 2.3GHz CPU with 8GByte RAM running macOS Catalina.
The polynomial-time explanation-redundancy check and a single
PI-explanation extraction proposed in Section~\ref{sec:xdt} are
implemented in Perl. (An implementation using the Glucose SAT solver
was instrumental in validating the results, but for the DTs
considered, it was in general slower by at least one order of
magnitude.)
Performance-wise, training DTs with IAI takes from 4sec to 2310sec
with the average run time per dataset being 70sec.
In contrast, the time spent on eliminating explanation-redundancy
is \emph{negligible}, taking from 0.026sec to 0.4sec per tree, with an 
average time of 0.06sec.
ITI runs much faster that IAI and takes from 0.1sec to 2sec with
0.1sec on average; the elimination of explanation redundancy is
slightly more time consuming than for IAI, taking from
0.025sec to 5.4sec with 0.29sec on average.
This slowdown results from DTs learned with ITI being deeper on
average, and features being tested multiple times.

The summary of results is detailed in Table~\ref{tab:res}.
For each dataset, the table reports the number of features and also
instances 
as \emph{\#F} and \emph{\#S}, respectively.
Thereafter, it shows tree statistics for IAI and ITI, namely, tree
depth \emph{D}, number of nodes \emph{\#N}, test accuracy \emph{\%A}
and number of paths \emph{\#P}.
The percentage of explanation-redundant paths is given as \emph{\%R}
while the percentage of data instances (measured for the \emph{entire}
feature space) covered by redundant paths is \emph{\%C}.
Focusing solely on the explanation-redundant paths, a single
PI-explanation is extracted and the average (min.~or max., resp.)
percentage of redundant literals per path is denoted by \emph{\%avg}
(\emph{\%m} and \emph{\%M}, resp.).
Observe that despite the shallowness of the trees produced by IAI and
ITI, for the majority of datasets and with a few exceptions, the paths
in trees trained by both tools exhibit significant
explanation-redundancy.
In particular, on average, 32.1\% (46.9\%, resp.) of paths are
explanation-redundant for the trees obtained by IAI (ITI, resp.).
For some DTs, obtained with either IAI and ITI, more than 85\% of tree
paths are redundant.
Also, redundant paths of the trees of IAI (ITI, resp.)
cover on average 20.1\% (37.7\%, resp.) of feature space.
Moreover, in some cases, up to 89\% and 98\% of the entire feature
space is covered by the redundant paths for IAI and ITI,
respectively.
This means that DTs produced by IAI and ITI are unable to provide a
user with a succinct explanation for the \emph{vast majority} of data
instances.
In addition, the average number of redundant literals in
redundant paths for both IAI and ITI varies from 16\% to
65\%, but for some DTs it exceeds 80\%.

To summarize, the numbers shown for the selected datasets and for the
state-of-the-art DT training tools contrast the common belief in the
inherent interpretability of decision tree classifiers.
Perhaps as importantly, the performance figures confirm that
the elimination of explanation-redundancy in the DTs produced with
available tools has negligible computational cost.

\section{Conclusions} \label{sec:conc}

Decision trees are most often associated with interpretability.
This paper shows that in some situations, paths in a decision tree may
include many literals that are irrelevant for an
explanation, and that this holds true even for irreducible decision
trees. 
Moreover, the paper proposes a linear time test to decide whether a
decision tree path contains irrelevant literals, and uses such test to
devise a polynomial time algorithm for computing one PI-explanation of
a decision tree.
Furthermore, the paper shows the connection between enumerating the
PI-explanations of DTs and the enumeration of minimal hitting sets.
Experimental results obtained on publicly available datasets, using
state-of-the-art decision tree learners, show that in practice induced 
paths in decision trees may contain irrelevant literals, even when the
decision tree is irreducible. For the decision trees considered in the
experiments, the run times of the proposed algorithms are either
negligible or comparable to tree learning times.

\bibliographystyle{abbrv}
\bibliography{refs}

\appendix

\section{Case Studies}

\subsection{Analysis of DT from Russel\&Norvig's Book}

\subsubsection{Decision Tree}
This case study considers the decision tree (DT) shown in
Figure~\ref{fig:dt-RN}, taken from
\cite[Ch.~18,page~702]{russell-bk10}.  The example consists in
deciding whether to wait for a table at a restaurant. Six features are
used in the DT, namely:
\begin{itemize}
\item Alternate: whether there is a suitable alternative restaurant nearby.
\item Bar: whether the restaurant has a comfortable bar area to wait in.
\item Fri/Sat: true on Fridays and Saturdays.
\item Hungry: whether the people are hungry.
\item Patrons: how many people are in the restaurant (values are None, Some, and Full).
\item Type: the kind of restaurant (French, Italian, Thai, or burger).
\end{itemize}

\begin{figure}[h]
  \centering \scalebox{0.95}{\forestset{
  BDT/.style={
    for tree={
      l=1.125cm,s sep=1.0cm,
      if n children=0{}{circle},
      draw,
      edge={
        my edge
      },
      if n=1{
        edge+={0 my edge},
      }{},
    }
  },
}
\begin{forest}
  BDT
  [\scriptsize{Patrons}
    [{\footnotesize\color{darkred}No}, edge label={node[near start,left,xshift=-1pt] {{\tiny{None}}}}] 
    [\scriptsize{Hungry}, edge label={node[near start,left,xshift=0.1pt] {{\tiny{Full}}}}  
    [{\footnotesize\color{darkred}No}, edge label={node[near start,left,xshift=-1pt] {{\tiny{No}}}}]
     [\scriptsize{Type}, edge label={node[near start,right,xshift=-1pt] {{\tiny{Yes}}}}
      [{\footnotesize\color{darkgreen}Yes}, edge label={node[near start,left,xshift=-1pt] {{\tiny{French}}}}]
       [{\footnotesize\color{darkred}No}, edge label={node[near end,right,xshift=-1pt] {{\tiny{Italian}}}}]
       [\scriptsize{Fri/Sat}, edge label={node[near start,right,xshift=-1pt] {{\tiny{Thai}}}}
            [{\footnotesize\color{darkred}No}, edge label={node[near start,left,xshift=-1pt] {{\tiny{No}}}}]
            [{\footnotesize\color{darkgreen}Yes}, edge label={node[near start,right,xshift=-1pt] {{\tiny{Yes}}}}]       
       ]
       [{\footnotesize\color{darkgreen}Yes}, edge label={node[near end,right,xshift=-1pt] {{\tiny{Burger}}}}]       
     ]  
    ]
    [{\footnotesize\color{darkgreen}Yes}, edge label={node[near start,right,xshift=-1pt] {{\tiny{Some}}}}]
  ]
\end{forest}

 }
  \caption{Example of decision tree  from Russel\&Norvig's Book.}
  \label{fig:dt-RN}
\end{figure}
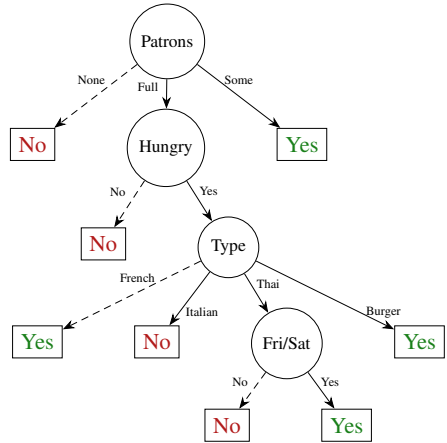

\subsubsection{Redundancy Analysis Results}
Analysis of the paths in the DT shown in Figure~\ref{fig:dt-RN}
yields the following results.
\begin{itemize}
\item path (\textit{Patrons}=\textit{None}) is explanation-irredundant.
\item path (\textit{Patrons}=\textit{Ful} and
  \textit{Hungry}=\textit{No}) is explanation-irredundant. 
\item path (\textit{Patrons}=\textit{Ful}  and
  \textit{Hungry}=\textit{Yes} and  \textit{Type}=\textit{Italian}) is
  explanation-redundant. If the values of \textit{Patrons} and
  \textit{Type} are fixed, then the value of \textit{Hungry} is
  irrelevant for the prediction. 
\item path (\textit{Patrons}=\textit{Ful} and
  \textit{Hungry}=\textit{Yes} and \textit{Type}=\textit{Thai} and
  \textit{Fri/Sat}=\textit{No}) is explanation-redundant.  If the
  values of \textit{Patrons}, \textit{Type} and \textit{Fri/Sat}  are
  fixed, then the value of \textit{Hungry} is irrelevant for the
  prediction.
\item path (\textit{Patrons}=\textit{Some}) is explanation-irredundant.
\item path (\textit{Patrons}=\textit{Ful}  and
  \textit{Hungry}=\textit{Yes}  and \textit{Type}=\textit{Frencg}) is
  explanation-irredundant. 
\item path (\textit{Patrons}=\textit{Ful} and
  \textit{Hungry}=\textit{Yes} and \textit{Type}=\textit{Thai} and
  \textit{Fri/Sat}=\textit{Yes}) is explanation-irredundant.
\item path (\textit{Patrons}=\textit{Ful} and
  \textit{Hungry}=\textit{Yes} and \textit{Type}=\textit{Burger}) is
  explanation-irredundant. 
\end{itemize}

As result,  2 out of 8 paths  exhibit explanation-redundancy. Thus, we
conclude that the DT exhibits $25$\% of explanation-redundancy. 

\subsection{Analysis of DT from Poole\&Mackworth's Book}

\subsubsection{Decision Tree}
This case study considers the decision tree shown in
Figure~\ref{fig:dt-PM}, taken from \cite[Ch.~07,page~298]{poole-bk17}.
The example consists in predicting whether a person reads an article
posted to a bulletin board given properties of the article. There are
three features in the decision tree:
\begin{itemize}
\item Author: whether the author is known or not.
\item Thread:  whether the article started a new thread or was a follow-up.
\item Length: the length of the article (short or long).
\end{itemize}

\begin{figure}[h]
  \centering \scalebox{0.95}{\forestset{
  BDT/.style={
    for tree={
      l=1.125cm,s sep=1.0cm,
      if n children=0{}{circle},
      draw,
      edge={
        my edge
      },
      if n=1{
        edge+={0 my edge},
      }{},
    }
  },
}
\begin{forest}
  BDT
  [\scriptsize{Length}
    [{\footnotesize\color{darkred}Skips}, edge label={node[near start,left,xshift=-1pt] {{\tiny{Long}}}}] 
    [\scriptsize{Thread}, edge label={node[near start,right,xshift=-1pt] {{\tiny{Short}}}} 
        [{\footnotesize\color{darkgreen}Reads}, edge label={node[near start,left,xshift=-1pt] {{\tiny{New}}}}]
        [\scriptsize{Author}, edge label={node[near start,right,xshift=-1pt] {{\tiny{Follow-up}}}} 
         [{\footnotesize\color{darkred}Skips}, edge label={node[near start,left,xshift=-1pt] {{\tiny{Unknown}}}}] 
         [{\footnotesize\color{darkgreen}Reads}, edge label={node[near start,right,xshift=-1pt] {{\tiny{Known}}}}]
        ]
    ]
  ]
\end{forest}

 }
  \caption{Example of decision tree  from Poole\&Mackworth's Book.}
  \label{fig:dt-PM}
\end{figure}
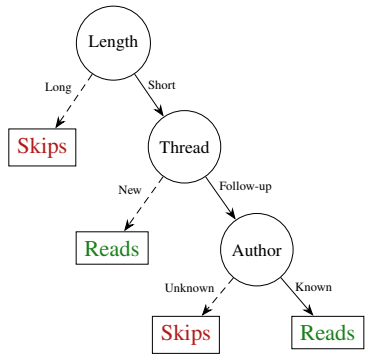

\subsubsection{Redundancy Analysis Results}
Analysis of the paths in the DT shown in Figure~\ref{fig:dt-PM}
yields the following results.
\begin{itemize}
\item path (\textit{Length}=\textit{long}) is explanation-irredundant.
\item path (\textit{Length}=\textit{short} and
  \textit{Thread}=\textit{follow-up} and
  \textit{Author}=\textit{unknown}) is explanation-redundant. If
  values of \textit{Thread} and \textit{Author} are fixed, then the
  value of \textit{Length} is irrelevant for the prediction. 
\item path (\textit{Length}=\textit{short} and
  \textit{Thread}=\textit{new}) is explanation-irredundant.
\item path (\textit{Length}=\textit{short} and
  \textit{Thread}=\textit{follow-up} and
  \textit{Author}=\textit{known}) is explanation-redundant. If the
  values of \textit{Length} and  \textit{Author} are  fixed, then the
  value of \textit{Thread} is irrelevant for the prediction. 
\end{itemize}

Accordingly,  2 out of 4 paths exhibit
explanation-redundancy. Therefore, we say that the DT  exhibit  $50$\%
of explanation-redundancy.


\subsection{Analysis of DT from Z.-H. Zhou's book}

\subsubsection{Decision Tree}
This case study considers the decision tree shown  in
Figure~\ref{fig:dt-Z}, taken from~\cite[Ch.~01,page~5]{zhou-bk12}.
The example consists in  predicting the type of a drawing, to be
chosen among the classes \emph{cross} and \emph{circle}. Two features
are used, namely:
\begin{itemize}[nosep]
\item $x>0.64\in\{\textsf{Y},\textsf{N}\}$.
\item $y>0.73\in\{\textsf{Y},\textsf{N}\}$.
\end{itemize}

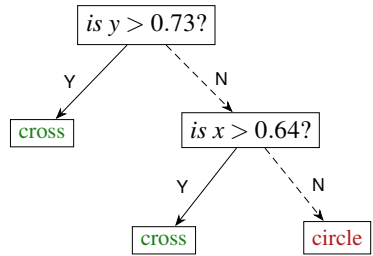
\begin{figure}[h]
  \centering \scalebox{0.95}{\forestset{
  BDT/.style={
    for tree={
      l=1.5cm,s sep=1.5cm,
      draw,
      edge={
        my edge
      },
      if n=2 {
        edge+={0 my edge},
      }{},
    }
  },
}
\begin{forest}
  BDT
  [\textit{is $y>0.73$}?
    [{\footnotesize\color{darkgreen}cross}, edge label={node[midway,left,xshift=-2.75pt] {{\scriptsize\textsf{Y}}}}]
    [\textit{is $x>0.64$}?, edge label={node[midway,right,xshift=2.75pt] {{\scriptsize\textsf{N}}}}
     [{\footnotesize\color{darkgreen}cross}, edge label={node[midway,left,xshift=-3.5pt] {{\scriptsize\textsf{Y}}}}]
     [{\footnotesize\color{darkred}circle}, edge label={node[midway,right,xshift=3.5pt] {{\scriptsize\textsf{N}}}}]
    ]
  ]
\end{forest}}
  \caption{Example of decision tree from Zhou's book~\cite{zhou-bk12}.}
  \label{fig:dt-Z}
\end{figure}

\subsubsection{Redundancy Analysis Results}
Analysis of the paths in the DT shown in Figure~\ref{fig:dt-Z} yields
the following results.
\begin{itemize}[itemsep=0pt,topsep=0pt,partopsep=0pt]
\item path ($y>0.73$) is explanation-irredundant.
\item path ($y\le0.73$ and $x>0.64$) is explanation-redundant. If the
  value of $y$ is fixed, then the value of $x$ is irrelevant for the
  prediction.
\item path ($y\le0.73$ and $x\le0.64$) is explanation-irredundant.
\end{itemize}

As a result, 1 out of 3 paths exhibit explanation-redundancy. Thus, we
say that the DT exhibits $33.33$\% of explanation-redundancy.

\subsection{Additional Examples}

It is interesting to note that the DTs used in a number of books and
surveys exhibit explanation-redundancy. A non-exhaustive list of
references
includes~\cite{moret-acmcs82,aha-ker97,rokach-tsmc05,rokach-bk07,russell-bk10,flach-bk12,kotsiantis-air13,alpaydin-bk16,bramer-bk16,solberg-nature16,poole-bk17}.

\section{Full Table of Results}

\autoref{tab:res2} presents the experimental results obtained on an
extended set of datasets.

\setlength{\tabcolsep}{5pt}
\rowcolors{2}{gray!10}{}
\begin{table*}[h]
\centering
\resizebox{\textwidth}{!}{
  \begin{tabular}{l>{\lpr}S[table-format=3.0,table-space-text-pre=\lpr]S[table-format=5.0,table-space-text-post=\rpr]<{\rpr}cS[table-format=3.0]S[table-format=3.0]S[table-format=2.0]S[table-format=2.0]S[table-format=2.0]cccS[table-format=2.0]S[table-format=4.0]S[table-format=3.0]S[table-format=4.0]S[table-format=2.0]S[table-format=2.0]S[table-format=2.0]cS[table-format=2.0]}
\toprule[1.2pt]
\rowcolor{white}
\multirow{2}{*}{\bf Dataset} & \multicolumn{2}{c}{\multirow{2}{*}{\bf (\#F\hspace{0.5cm} \#S)}} & \multicolumn{9}{c}{\bf IAI} & \multicolumn{9}{c}{\bf ITI} \\
  \cmidrule[0.8pt](lr{.75em}){4-12}
  \cmidrule[0.8pt](lr{.75em}){13-21}
\rowcolor{white}
& \multicolumn{2}{c}{} & {\bf D} & {\bf \#N} & {\bf \%A} & {\bf \#P} & {\bf \%R} & {\bf \%C} & {\bf \%m} & {\bf \%M} & {\bf \%avg} & {\bf D} & {\bf \#N} & {\bf \%A} & {\bf \#P} & {\bf \%R} & {\bf \%C} & {\bf \%m} & {\bf \%M} & {\bf \%avg} \\
\toprule[1.2pt]

IndiansDiabetes & 8 & 768 & 6 & 33 & 67 & 17 & 35 & 2 & 20 & 25 & 21 & 21 & 67 &  60 & 34 & 38 & 22 & 20 & 33 & 26  \\ 

adult\_data & 12 & 6061 & 6 & 83 & 78 & 42 & 33 & 25 & 20 & 40 & 25 & 17 & 509 &  73 & 255 & 75 & 91 & 10 & 66 & 22  \\ 

anneal & 38 & 886 & 6 & 29 & 99 & 15 & 26 & 16 & 16 & 33 & 21 & 9 & 31 &  100 & 16 & 25 & 4 & 12 & 20 & 16  \\ 

appendicitis & 7 & 106 & 2 & 7 & 68 & 4 & 0 & 0 & \textemdash & \textemdash & \textemdash & 3 & 7 &  81 & 4 & 50 & 3 & 50 & 66 & 58  \\ 

australian & 14 & 690 & 6 & 45 & 61 & 23 & 17 & 8 & 20 & 33 & 23 & 7 & 33 &  82 & 17 & 35 & 51 & 16 & 33 & 26  \\ 

auto & 25 & 202 & 6 & 33 & 53 & 17 & 23 & 1 & 16 & 33 & 23 & 10 & 47 &  75 & 24 & 33 & 53 & 14 & 40 & 26  \\ 

backache & 32 & 180 & 4 & 17 & 72 & 9 & 33 & 39 & 25 & 33 & 30 & 3 & 9 &  91 & 5 & 80 & 87 & 50 & 66 & 54  \\ 

balance & 4 & 625 & 6 & 93 & 81 & 47 & 61 & 41 & 25 & 50 & 27 & 12 & 105 &  81 & 53 & 50 & 62 & 25 & 50 & 26  \\ 

bank & 19 & 36293 & 6 & 113 & 88 & 57 & 5 & 12 & 16 & 20 & 18 & 19 & 1467 &  86 & 734 & 69 & 64 & 7 & 63 & 27  \\ 

banknote & 4 & 1348 & 3 & 9 & 67 & 5 & 20 & 0 & 66 & 66 & 66 & 35 & 71 &  54 & 36 & 83 & 2 & 33 & 50 & 43  \\ 

biodegradation & 41 & 1052 & 5 & 19 & 65 & 10 & 30 & 1 & 25 & 50 & 33 & 8 & 71 &  76 & 36 & 50 & 8 & 14 & 40 & 21  \\ 

biomed & 8 & 209 & 3 & 9 & 66 & 5 & 20 & 1 & 50 & 50 & 50 & 8 & 33 &  64 & 17 & 29 & 9 & 25 & 50 & 33  \\ 

breast-cancer & 9 & 272 & 6 & 61 & 74 & 31 & 41 & 13 & 16 & 33 & 20 & 7 & 25 &  70 & 13 & 23 & 33 & 25 & 33 & 30  \\ 

bupa & 6 & 341 & 6 & 47 & 62 & 24 & 29 & 5 & 20 & 33 & 25 & 24 & 49 &  59 & 25 & 64 & 25 & 16 & 50 & 25  \\ 

cancer & 9 & 449 & 6 & 37 & 87 & 19 & 36 & 9 & 20 & 25 & 21 & 5 & 21 &  84 & 11 & 54 & 10 & 25 & 50 & 37  \\ 

car & 6 & 1728 & 6 & 43 & 96 & 22 & 86 & 89 & 20 & 80 & 45 & 11 & 57 &  98 & 29 & 65 & 41 & 16 & 50 & 30  \\ 

cars & 8 & 392 & 2 & 5 & 100 & 3 & 0 & 0 & \textemdash & \textemdash & \textemdash & 14 & 45 &  87 & 23 & 26 & 1 & 20 & 40 & 26  \\ 

cleveland-nominal & 7 & 187 & 6 & 81 & 31 & 41 & 24 & 15 & 16 & 33 & 20 & 7 & 65 &  31 & 33 & 33 & 71 & 16 & 40 & 30  \\ 

cloud & 7 & 108 & 3 & 9 & 31 & 5 & 0 & 0 & \textemdash & \textemdash & \textemdash & 10 & 21 &  36 & 11 & 18 & 91 & 25 & 25 & 25  \\ 

colic & 22 & 357 & 6 & 55 & 81 & 28 & 46 & 6 & 16 & 33 & 20 & 4 & 17 &  80 & 9 & 33 & 27 & 25 & 25 & 25  \\ 

compas & 11 & 1155 & 6 & 77 & 34 & 39 & 17 & 8 & 16 & 20 & 17 & 15 & 183 &  37 & 92 & 66 & 43 & 12 & 60 & 27  \\ 

contraceptive & 9 & 1425 & 6 & 99 & 49 & 50 & 8 & 2 & 20 & 60 & 37 & 17 & 385 &  48 & 193 & 27 & 32 & 12 & 66 & 21  \\ 

corral & 6 & 64 & 6 & 19 & 92 & 10 & 80 & 50 & 20 & 66 & 42 & 4 & 13 &  100 & 7 & 71 & 50 & 33 & 50 & 40  \\ 

dermatology & 34 & 366 & 6 & 33 & 90 & 17 & 23 & 3 & 16 & 33 & 21 & 7 & 17 &  95 & 9 & 22 & 0 & 14 & 20 & 17  \\ 

divorce & 54 & 150 & 5 & 15 & 90 & 8 & 50 & 19 & 20 & 33 & 24 & 2 & 5 &  96 & 3 & 33 & 16 & 50 & 50 & 50  \\ 

ecoli & 7 & 327 & 6 & 45 & 75 & 23 & 4 & 5 & 20 & 20 & 20 & 53 & 109 &  59 & 55 & 25 & 19 & 20 & 66 & 36  \\ 

german\_data & 21 & 1000 & 6 & 25 & 61 & 13 & 38 & 10 & 20 & 40 & 29 & 10 & 99 &  72 & 50 & 46 & 13 & 12 & 40 & 22  \\ 

glass & 9 & 204 & 6 & 35 & 53 & 18 & 0 & 0 & \textemdash & \textemdash & \textemdash & 27 & 65 &  48 & 33 & 12 & 3 & 25 & 33 & 27  \\ 

glass2 & 9 & 162 & 3 & 11 & 66 & 6 & 0 & 0 & \textemdash & \textemdash & \textemdash & 7 & 15 &  51 & 8 & 75 & 96 & 20 & 60 & 41  \\ 

haberman & 3 & 289 & 6 & 55 & 58 & 28 & 14 & 4 & 33 & 33 & 33 & 12 & 35 &  74 & 18 & 44 & 21 & 33 & 50 & 39  \\ 

hayes-roth & 4 & 93 & 6 & 23 & 78 & 12 & 25 & 32 & 50 & 66 & 61 & 6 & 17 &  78 & 9 & 22 & 32 & 50 & 66 & 58  \\ 

heart-c & 13 & 302 & 6 & 43 & 65 & 22 & 36 & 18 & 20 & 33 & 22 & 4 & 15 &  75 & 8 & 87 & 81 & 25 & 50 & 34  \\ 

heart-h & 13 & 293 & 6 & 37 & 59 & 19 & 31 & 4 & 20 & 40 & 24 & 8 & 25 &  77 & 13 & 61 & 60 & 20 & 50 & 32  \\ 

heart-statlog & 13 & 270 & 6 & 33 & 55 & 17 & 29 & 5 & 20 & 33 & 30 & 4 & 13 &  81 & 7 & 71 & 31 & 25 & 50 & 35  \\ 

hepatitis & 19 & 155 & 5 & 17 & 77 & 9 & 33 & 6 & 20 & 33 & 24 & 3 & 11 &  80 & 6 & 33 & 14 & 33 & 33 & 33  \\ 

house-votes-84 & 16 & 298 & 6 & 49 & 91 & 25 & 68 & 67 & 16 & 50 & 27 & 3 & 9 &  90 & 5 & 80 & 75 & 33 & 50 & 41  \\ 

hungarian & 13 & 293 & 6 & 33 & 69 & 17 & 29 & 2 & 25 & 40 & 31 & 4 & 19 &  77 & 10 & 40 & 30 & 33 & 33 & 33  \\ 

ionosphere & 34 & 350 & 4 & 9 & 70 & 5 & 60 & 3 & 33 & 50 & 38 & 5 & 17 &  80 & 9 & 33 & 0 & 33 & 60 & 47  \\ 

iris & 4 & 149 & 5 & 23 & 90 & 12 & 41 & 25 & 25 & 33 & 30 & 10 & 21 &  63 & 11 & 36 & 13 & 50 & 50 & 50  \\ 

irish & 5 & 470 & 4 & 13 & 97 & 7 & 71 & 54 & 33 & 50 & 36 & 3 & 7 &  100 & 4 & 0 & 0 & \textemdash & \textemdash & \textemdash  \\ 

kr-vs-kp & 36 & 3196 & 6 & 49 & 96 & 25 & 80 & 75 & 16 & 60 & 33 & 13 & 67 &  99 & 34 & 79 & 43 & 7 & 70 & 35  \\ 

lending\_data & 9 & 5082 & 6 & 45 & 73 & 23 & 73 & 80 & 16 & 50 & 25 & 14 & 507 &  65 & 254 & 69 & 80 & 12 & 75 & 25  \\ 

letter & 16 & 18668 & 6 & 127 & 58 & 64 & 1 & 0 & 20 & 20 & 20 & 46 & 4857 &  68 & 2429 & 6 & 7 & 6 & 25 & 9  \\ 

lupus & 3 & 87 & 2 & 7 & 44 & 4 & 0 & 0 & \textemdash & \textemdash & \textemdash & 1 & 3 &  61 & 2 & 100 & 100 & 100 & 100 & 100  \\ 

lymphography & 18 & 148 & 6 & 61 & 76 & 31 & 35 & 25 & 16 & 33 & 21 & 6 & 21 &  86 & 11 & 9 & 0 & 16 & 16 & 16  \\ 

messidor & 19 & 1146 & 3 & 7 & 50 & 4 & 50 & 1 & 50 & 66 & 58 & 22 & 107 &  52 & 54 & 88 & 99 & 10 & 57 & 28  \\ 

meteo & 4 & 14 & 5 & 13 & 33 & 7 & 42 & 33 & 33 & 50 & 38 & 1 & 3 &  33 & 2 & 0 & 0 & \textemdash & \textemdash & \textemdash  \\ 

molecular-biology\_promoters & 58 & 106 & 6 & 17 & 86 & 9 & 33 & 19 & 16 & 33 & 23 & 3 & 9 &  81 & 5 & 20 & 14 & 33 & 33 & 33  \\ 

monk1 & 6 & 124 & 4 & 17 & 100 & 9 & 66 & 41 & 25 & 50 & 36 & 5 & 13 &  92 & 7 & 57 & 41 & 20 & 40 & 29  \\ 

monk2 & 6 & 169 & 6 & 67 & 82 & 34 & 64 & 49 & 16 & 66 & 32 & 5 & 25 &  50 & 13 & 23 & 20 & 25 & 33 & 30  \\ 

monk3 & 6 & 122 & 6 & 35 & 80 & 18 & 61 & 36 & 20 & 60 & 37 & 2 & 5 &  92 & 3 & 33 & 16 & 50 & 50 & 50  \\ 

mortality & 118 & 13442 & 6 & 111 & 74 & 56 & 8 & 14 & 16 & 20 & 17 & 26 & 865 &  76 & 433 & 61 & 61 & 7 & 54 & 19  \\ 

mouse & 5 & 57 & 3 & 9 & 83 & 5 & 20 & 0 & 33 & 33 & 33 & 2 & 5 &  75 & 3 & 0 & 0 & \textemdash & \textemdash & \textemdash  \\ 

mushroom & 22 & 8124 & 6 & 39 & 100 & 20 & 80 & 44 & 16 & 33 & 24 & 5 & 23 &  100 & 12 & 50 & 31 & 20 & 40 & 25  \\ 

mux6 & 6 & 64 & 6 & 55 & 61 & 28 & 85 & 78 & 20 & 50 & 37 & 4 & 15 &  46 & 8 & 37 & 31 & 25 & 33 & 30  \\ 

new-thyroid & 5 & 215 & 3 & 11 & 95 & 6 & 33 & 4 & 33 & 33 & 33 & 14 & 29 &  79 & 15 & 26 & 5 & 20 & 50 & 30  \\ 

pendigits & 16 & 10992 & 6 & 121 & 88 & 61 & 0 & 0 & \textemdash & \textemdash & \textemdash & 38 & 937 &  85 & 469 & 25 & 86 & 6 & 25 & 11  \\ 

postoperative-patient-data & 8 & 78 & 6 & 43 & 50 & 22 & 59 & 44 & 16 & 50 & 25 & 6 & 15 &  56 & 8 & 75 & 25 & 16 & 50 & 36  \\ 

primary-tumor & 15 & 228 & 6 & 55 & 71 & 28 & 35 & 21 & 16 & 33 & 21 & 6 & 21 &  82 & 11 & 54 & 31 & 16 & 50 & 31  \\ 

promoters & 58 & 106 & 1 & 3 & 90 & 2 & 0 & 0 & \textemdash & \textemdash & \textemdash & 3 & 9 &  81 & 5 & 20 & 14 & 33 & 33 & 33  \\ 

recidivism\_data & 15 & 3998 & 6 & 105 & 61 & 53 & 28 & 22 & 16 & 33 & 18 & 15 & 611 &  51 & 306 & 53 & 38 & 9 & 44 & 16  \\ 

schizo & 14 & 340 & 6 & 17 & 55 & 9 & 55 & 5 & 50 & 66 & 58 & 13 & 51 &  52 & 26 & 30 & 64 & 20 & 25 & 21  \\ 

segmentation & 19 & 210 & 4 & 15 & 38 & 8 & 0 & 0 & \textemdash & \textemdash & \textemdash & 27 & 57 &  23 & 29 & 48 & 2 & 11 & 71 & 39  \\ 

seismic\_bumps & 18 & 2578 & 6 & 37 & 89 & 19 & 42 & 19 & 20 & 33 & 24 & 8 & 39 &  93 & 20 & 60 & 79 & 20 & 60 & 42  \\ 

shuttle & 9 & 58000 & 6 & 63 & 99 & 32 & 28 & 7 & 20 & 33 & 23 & 23 & 159 &  99 & 80 & 33 & 9 & 14 & 50 & 30  \\ 

soybean & 35 & 623 & 6 & 63 & 88 & 32 & 9 & 5 & 25 & 25 & 25 & 16 & 71 &  89 & 36 & 22 & 1 & 9 & 12 & 10  \\ 

spambase & 57 & 4210 & 6 & 63 & 75 & 32 & 37 & 12 & 16 & 33 & 19 & 15 & 143 &  91 & 72 & 76 & 98 & 7 & 58 & 25  \\ 

spect & 22 & 228 & 6 & 45 & 82 & 23 & 60 & 51 & 20 & 50 & 35 & 6 & 15 &  86 & 8 & 87 & 98 & 50 & 83 & 65  \\ 

splice & 2 & 3178 & 3 & 7 & 50 & 4 & 0 & 0 & \textemdash & \textemdash & \textemdash & 88 & 177 &  55 & 89 & 0 & 0 & \textemdash & \textemdash & \textemdash  \\ 

student-mat & 32 & 395 & 6 & 109 & 35 & 55 & 9 & 3 & 20 & 25 & 21 & 22 & 177 &  25 & 89 & 6 & 6 & 11 & 20 & 15  \\ 

student-por & 32 & 649 & 6 & 119 & 30 & 60 & 1 & 0 & 20 & 20 & 20 & 22 & 259 &  26 & 130 & 9 & 5 & 7 & 20 & 10  \\ 

tae & 5 & 110 & 6 & 43 & 40 & 22 & 9 & 6 & 25 & 33 & 29 & 8 & 23 &  36 & 12 & 41 & 21 & 25 & 40 & 31  \\ 

titanic & 3 & 24 & 2 & 5 & 40 & 3 & 33 & 25 & 50 & 50 & 50 & 2 & 5 &  20 & 3 & 0 & 0 & \textemdash & \textemdash & \textemdash  \\ 

tram\_2000\_side\_16x16 & 256 & 2000 & 1 & 3 & 100 & 2 & 0 & 0 & \textemdash & \textemdash & \textemdash & 1 & 3 &  100 & 2 & 0 & 0 & \textemdash & \textemdash & \textemdash  \\ 

uci\_mammo\_data & 13 & 126 & 6 & 53 & 11 & 27 & 51 & 43 & 16 & 40 & 24 & 9 & 23 &  38 & 12 & 66 & 25 & 25 & 55 & 39  \\ 

vehicle & 18 & 846 & 6 & 79 & 49 & 40 & 10 & 0 & 16 & 20 & 19 & 24 & 141 &  58 & 71 & 45 & 54 & 9 & 41 & 23  \\ 

wdbc & 30 & 569 & 2 & 7 & 87 & 4 & 0 & 0 & \textemdash & \textemdash & \textemdash & 57 & 115 &  61 & 58 & 94 & 10 & 14 & 73 & 49  \\ 

wpbc & 33 & 198 & 2 & 5 & 57 & 3 & 33 & 0 & 50 & 50 & 50 & 2 & 5 &  72 & 3 & 100 & 100 & 100 & 100 & 100  \\ 

yeast & 9 & 1462 & 6 & 45 & 49 & 23 & 4 & 0 & 25 & 25 & 25 & 64 & 493 &  37 & 247 & 14 & 0 & 11 & 25 & 16  \\ 

zoo & 16 & 59 & 6 & 23 & 91 & 12 & 33 & 7 & 16 & 33 & 21 & 6 & 13 &  83 & 7 & 0 & 0 & \textemdash & \textemdash & \textemdash  \\ 

\bottomrule[1.2pt]
\end{tabular}
}
\caption{Explanation-redundancy  in decision trees obtained with IAI and ITI. \label{tab:res2}}
\end{table*}

\end{document}